\documentclass[letterpaper, 10 pt, journal, twoside]{IEEEtran}

\IEEEoverridecommandlockouts

\usepackage[dvipsnames]{xcolor}

\usepackage{amsthm}
\usepackage[T1]{fontenc}
\usepackage[utf8]{inputenc}
\usepackage{times}
\usepackage{multicol}
\usepackage{siunitx}
\usepackage[maxbibnames=99,maxcitenames=2,natbib=true,style=numeric-comp,backend=biber,sorting=none,giveninits=true]{biblatex}
\usepackage[linesnumbered,ruled,vlined,commentsnumbered,noend]{algorithm2e}
\usepackage{mathrsfs}
\usepackage{mathtools}
\usepackage{soul}
\usepackage{prelude}

\newcommand{\myTitle}{%
    \texorpdfstring{%
        Task and Motion Informed Trees (TMIT*): Almost-Surely Asymptotically Optimal Integrated Task and Motion Planning%
        }{%
        Task and Motion Informed Trees (TMIT*): Almost-Surely Asymptotically Optimal Integrated Task and Motion Planning%
        }%
    }

\hypersetup{%
    pdftitle={\myTitle{}},
    pdfauthor={Wil Thomason, Marlin P. Strub, and Jonathan D. Gammell},
    pdfkeywords={Integrated Task and Motion Planning;TAMP;TMP;Asymptotically Optimal Motion Planning;Motion Planning},
    pdfsubject={Integrated Task and Motion Planning},
    colorlinks=true,
    linkcolor={red!50!black},
    citecolor={green!80!black},
    urlcolor={blue!80!black},
}%

\newbibmacro{string+url}[1]{%
  \iffieldundef{url}{#1}{\href{\thefield{url}}{#1}}}
\DeclareFieldFormat{title}{\usebibmacro{string+url}{\mkbibemph{#1}}}
\DeclareFieldFormat[article]{title}{\usebibmacro{string+url}{\mkbibquote{#1}}}
\DeclareFieldFormat[inproceedings]{title}{\usebibmacro{string+url}{\mkbibquote{#1}}}
\DeclareFieldFormat[book]{title}{\usebibmacro{string+url}{\mkbibquote{#1}}}
\DeclareFieldFormat[incollection]{title}{\usebibmacro{string+url}{\mkbibquote{#1}}}
\crefname{line}{line}{lines}
\crefname{figure}{Fig.}{Figs.}
\Crefname{figure}{Fig.}{Figs.}
\crefname{equation}{Eq.}{Eqs.}
\Crefname{equation}{Eq.}{Eqs.}
\crefname{section}{Sec.}{Secs.}
\Crefname{section}{Sec.}{Secs.}
\crefname{definition}{Def.}{Defs.}
\Crefname{definition}{Def.}{Defs.}
\crefname{algorithm}{Alg.}{Algs.}
\Crefname{algorithm}{Alg.}{Algs.}

\newtheoremstyle{prettydef}%
{6pt}%
{6pt}%
{}%
{}%
{\bfseries}%
{}%
{\newline}%
{\thmname{#1}\thmnumber{ #2}:\thmnote{ \textnormal{\textit{#3}}}}

\theoremstyle{prettydef}
\newtheorem{definition}{Definition}
\newtheorem{theorem}{Theorem}

\newbool{comments}
\ifbool{comments}{\usepackage{todonotes}}{}

\usetikzlibrary{arrows, decorations.markings, positioning, shapes, calc}
\tikzstyle{vecArrow} = [thick, decoration={markings,mark=at position
		1 with {\arrow[semithick]{open triangle 60}}},
double distance=1.4pt, shorten >= 5.5pt,
preaction = {decorate},
postaction = {draw,line width=1.4pt, white,shorten >= 4.5pt}]
\tikzstyle{innerWhite} = [semithick, white,line width=1.4pt, shorten >= 4.5pt]
\tikzstyle{io} = [trapezium, trapezium left angle=70, trapezium right angle=110, text centered, draw=black, align=center]
\tikzstyle{process} = [rectangle, text centered, draw=black, align=center]
\tikzstyle{stepnode} = [circle, text centered, draw=black, align=center]
\tikzstyle{decision} = [diamond, text centered, draw=black, align=center, aspect=2.5]
\tikzstyle{arrow} = [thick, ->, >=stealth]

\let\vec\boldsymbol
\setlist[enumerate]{label=({\arabic*})}

\newcommand{\true}{\texttt{true}}
\newcommand{\false}{\texttt{false}}

\newcommand{\SE}[1]{\ensuremath{\mathrm{SE}(#1)}}
\newcommand{\BoolDom}{\ensuremath{\mathbb{B}}}
\newcommand{\ModeFam}{\ensuremath{\mathbb{M}}}
\newcommand{\mode}{\ensuremath{\mathcal{M}}}

\newcommand{\robot}{\ensuremath{R}}

\newcommand{\objects}{\ensuremath{\mathcal{O}}}

\newcommand{\constraint}{\ensuremath{\phi}}
\newcommand{\effect}{\ensuremath{\psi}}

\newcommand{\action}{\ensuremath{\alpha}}
\newcommand{\actions}{\ensuremath{\mathcal{A}}}

\newcommand{\state}{\ensuremath{\vec{q}}}
\newcommand{\statespace}{\ensuremath{\mathcal{Q}}}
\newcommand{\config}{\ensuremath{\mathcal{Q}}}

\newcommand{\sym}{\ensuremath{s}}
\newcommand{\symbols}{\ensuremath{\mathcal{S}}}
\newcommand{\predicates}{\ensuremath{\mathcal{P}}}
\newcommand{\geometricpreds}{\ensuremath{\predicates_\text{G}}}
\newcommand{\discretepreds}{\ensuremath{\predicates_\text{D}}}

\newcommand{\motion}{\ensuremath{\sigma}}
\newcommand{\motions}{\ensuremath{\Sigma}}
\newcommand{\cost}{\ensuremath{\gamma}}

\newcommand{\bvar}[1]{\ensuremath{a_{#1}}}

\newcommand{\new}[1]{#1}

\newcommand{\squeezeWords}{\looseness=-1}

\begin{document}

\title{\myTitle{}}

\author{Wil Thomason$^{1}$, Marlin P. Strub$^{2}$, and Jonathan D. Gammell$^{3}$%
	\thanks{Manuscript received: February 24, 2022; Revised: July 19, 2022; Accepted: July 23, 2022.}%
	\thanks{This paper was recommended for publication by Editor Hanna Kurniawati upon evaluation of the Associate Editor and Reviewers’ comments.}%
	\thanks{This work was supported by the National Defense Science \& Engineering Graduate Fellowship and UK Research and Innovation/EPSRC through ACE-OPS:\@ From Autonomy to Cognitive assistance in Emergency OPerationS [EP/S030832/1].
		We are grateful for this support.
		For the purpose of open access, the authors have applied a Creative Commons Attribution (CC BY) licence to any Author Accepted Manuscript version arising from this submission.}%
	\thanks{$^{1}$Department of Computer Science, Rice University, Houston, Texas, USA.\@ \texttt{\small wbthomason@rice.edu}}%
	\thanks{$^{2}$Jet Propulsion Laboratory, California Institute of Technology.\@ Work done at University of Oxford. \texttt{\small marlin.p.strub@jpl.nasa.gov}}%
	\thanks{$^{3}$Estimation, Search, and
		Planning (ESP) Group, Oxford Robotics Institute, University of
		Oxford, UK.\@ \texttt{\small gammell@robots.ox.ac.uk}}%
	\thanks{Digital Object Identifier (DOI): see top of this page.}}

\maketitle

\begin{abstract}
High-level autonomy requires discrete and continuous reasoning to decide both what actions to take and how to execute them.
Integrated Task and Motion Planning (TMP) algorithms solve these hybrid problems jointly to consider constraints between the discrete \new{symbolic} actions (i.e., the \emph{task plan}) and their continuous \new{geometric} realization (i.e., \emph{motion plans}).
This joint approach solves more difficult problems than approaches that address the task and motion subproblems independently.\squeezeWords

TMP algorithms combine and extend results from both task and motion planning.
TMP has mainly focused on computational performance and completeness and less on solution optimality.
Optimal TMP is difficult because the independent optima of the subproblems may not be the optimal \emph{integrated} solution, which can only be found by jointly optimizing both plans.

This paper presents Task and Motion Informed Trees (TMIT*), an optimal TMP algorithm that combines results from makespan-optimal task planning and almost-surely asymptotically optimal motion planning.
TMIT* interleaves asymmetric forward and reverse searches to delay computationally expensive operations until necessary and perform an efficient informed search directly in the problem’s hybrid state space.
This allows it to solve problems quickly and then converge towards the optimal solution with additional computational time, as demonstrated on the evaluated robotic-manipulation benchmark problems.

\end{abstract}

\begin{IEEEkeywords}
	Task and Motion Planning, Motion and Path Planning, Manipulation Planning, AI-Based Methods
\end{IEEEkeywords}

\section{Introduction}

\IEEEPARstart{P}{lanning} solutions to problems described by high-level specifications requires autonomously deciding both \emph{what} to do (i.e., the sequence of high-level actions) and \emph{how} to do it (i.e., the associated motions).
This is difficult since both of these decisions can affect later stages of the planning problem by altering the \new{valid} and reachable subsets of the search space.
Integrated Task and Motion Planning (TMP) is a holistic approach to solve these high-level planning problems by jointly considering the symbolic (i.e., actions) and geometric (i.e., motion) constraints on the solution.


Solving TMP problems is computationally expensive.
Evaluating a candidate sequence of actions (i.e., a \emph{task} or \emph{symbolic} plan) requires the computationally expensive operations of finding compatible action parameters and associated valid motion plans.
TMP algorithms typically consider multiple symbolic plans to solve a problem~\cite{dantam_incremental_constraint-based_2018} and must be efficient because the set of possible symbolic plans is combinatorially large for most real-world scenarios \new{and the sets of possible action parameters and motion plans are uncountably infinite}.

\begin{figure}[t!]
	\centering
	\includegraphics[clip,width=\columnwidth]{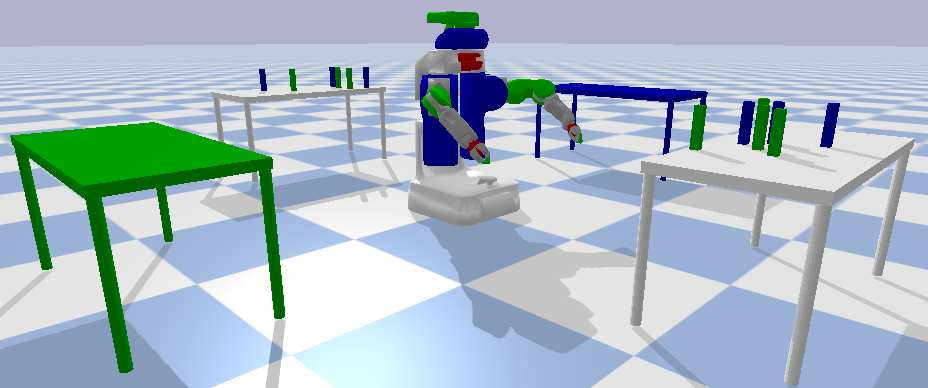}\caption{A clutter clearing example.
		The robot must move the green and blue sticks from their initial positions scattered on two tables to ending positions on the corresponding colored tables.
		Sticks may occlude others, forcing the robot to reason about a geometrically feasible order of manipulation.}\label{fig:clutter.scene}%
\end{figure}%
\begin{figure}[t!]
	\centering
	\includegraphics[clip,width=\columnwidth]{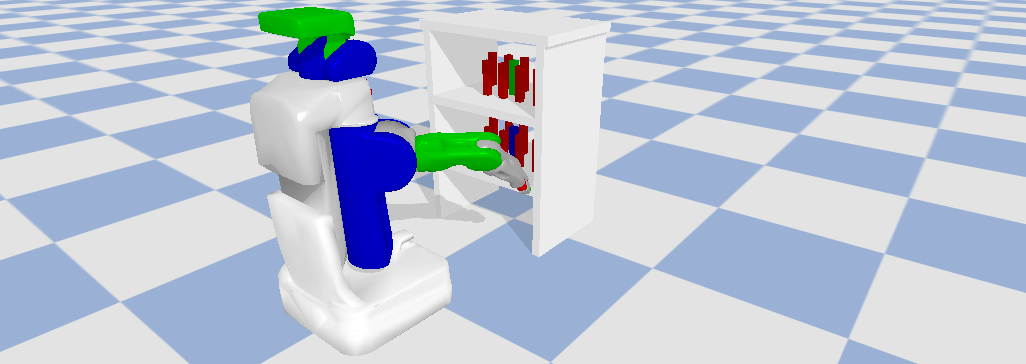}\caption{A shelf rearrangement example.
		The robot must move the green stick from the upper shelf to the lower shelf and the blue stick from the lower shelf to the upper shelf.
		The red sticks block many grasps of the two target sticks and must be maneuvered around or moved to solve the problem.}\label{fig:shelf.scene}%
\end{figure}

The majority of TMP work is focused on improving the computational performance of algorithms instead of finding optimal solutions~\cite{garrett_integrated_task_2020}.
Optimal TMP requires joint optimization of candidate symbolic plans and the corresponding action parameters and motion plans, which is more difficult than the independent optimal symbolic and optimal motion planning problems.
Existing solution-optimal TMP algorithms~\cite{toussaint_logic-geometric_programming_2015,vega-brown_asymptotically_optimal_2016,shome_pushing_boundaries_2020} have poor initial solution performance, require problem specific samplers and local motion planners, and/or are specific to manipulation planning.\squeezeWords

%
This paper presents Task and Motion Informed Trees (TMIT*), an optimal sampling-based TMP algorithm that extends results from makespan-optimal symbolic planning and almost-surely asymptotically optimal motion planning.
TMIT* efficiently searches the high-level problem's hybrid state space, \new{which consists of both discrete and continuously valued dimensions.}
This direct search allows TMIT* to be informed by both the symbolic and \new{geometric} constraints and to reuse motion planning effort as symbolic plans change.

TMIT* interleaves asymmetric forward and reverse searches~\cite{strub_aitstar_eitstar_2021} to identify geometrically infeasible plans and avoid computationally expensive action-parameter sampling and motion-planning operations.
When combined with novel SMT-based (Satisfiability Modulo Theories~\cite{biere_handbook_satisfiability_2009}) symbolic planning and a differentiable distance-based predicate representation~\cite{thomason_unified_sampling-based_2019}, this allows TMIT* to quickly find initial solutions to high-level problems and almost-surely converge asymptotically to the optimum with additional computational time.
We demonstrate the benefits of this approach on robotic-manipulation benchmark problems (\cref{fig:clutter.scene,fig:shelf.scene}), where TMIT* significantly outperforms earlier nonoptimal TMP work in initial solution times and is able to decrease solution costs given additional computational time.

\new{This work contributes methods for
	\begin{enumerate*}[label=(\arabic*)]
    \item almost-surely asymptotically optimal TMP that integrates relaxed SMT-based symbolic planning and sampling-based motion planning,
		\item anytime approximation of action-precondition-satisfying state manifolds to avoid state discretization,
		\item motion guided deferral of expensive precondition satisfying state sampling, and
		\item adaptive prioritization to order the search of a multimodal space
	\end{enumerate*}.\squeezeWords}


\section{Background and Related Work}
TMP is an active area of research drawing from advancements in the motion planning and task planning literatures.
TMIT* frames TMP as multimodal motion planning (\cref{sec:relatedwork.multimodal}).
It offers almost-surely asymptotically optimal TMP (\cref{sec:relatedwork.optimal}) with practical performance by extending work on batch-sampling-based motion planning (\cref{sec:relatedwork.batch}) and incremental symbolic planning (\cref{sec:relatedwork.symbolic}).
TMIT* uses deferred sampling of continuous action parameters to improve its computational efficiency (\cref{sec:relatedwork.deferred}).

\subsection{Multimodal Motion Planning and TMP}\label{sec:relatedwork.multimodal}

Multimodal motion planning finds valid motions through sequences of \emph{modes}, or discrete variations of a continuous configuration space, each imposing different constraints on the valid configurations~\cite{kingston_informing_multi-modal_2020,hauser_multi-modal_motion_2010}.
Manipulation planning is often modeled as multimodal motion planning~\cite{cambon_hybrid_approach_2009,hauser_multi-modal_motion_2010,alami_geometrical_approach_1989,hauser_randomized_multi-modal_2011,barry_manipulation_multiple_2013}, where modes correspond to choices of object grasps and placements.
TMP problems also have multimodal structure~\cite{hauser_randomized_multi-modal_2011,thomason_unified_sampling-based_2019,plaku_motion_planning_2010} with modes corresponding to different symbolic states that constrain the continuous \new{(i.e., geometric) state components}.
\new{Mode transitions are defined by symbolic high-level actions and link the task planning and motion planning subproblems.}
Framing TMP as multimodal motion planning allows approaches to avoid the backtracking necessary in most other \new{formulations}.\@\squeezeWords

\subsection{Almost-Surely Asymptotically Optimal TMP}\label{sec:relatedwork.optimal}

Most TMP work focuses on efficiently finding initial solutions.
Cost-aware or optimal TMP has only recently become a topic of interest to the field~\cite{vega-brown_asymptotically_optimal_2016}.
Earlier optimal TMP work~\cite{toussaint_logic-geometric_programming_2015} frames TMP as a nonlinear optimization problem and globally optimizes choices of action parameters and motions.
This optimization-based approach does not scale well with plan length, action space size, and geometric complexity.
\Citet{schmitt_optimal_sampling-based_2017} offer almost-surely asymptotically optimal manipulation planning based on a precomputed roadmap and domain-specific samplers for mode transitions.
\Citet{garrett_pddlstream_integrating_2020} produce TMP solutions that are optimal in symbolic action cost, but assume domain-specific samplers and do not jointly optimize their motions and action parameter choices.
Recent results~\cite{shome_pushing_boundaries_2020} have shown that almost-sure asymptotic optimality results from pure motion planning are preserved for multimodal TMP under realistic assumptions on the measure of mode transition sets.\squeezeWords

TMIT* uses almost-surely asymptotically optimal sampling based motion planning for performant planning in high dimensional state spaces.
It does not require any precomputation or domain-specific samplers, and solves both manipulation and more general TMP problems.

\subsection{Batch-Sampling-Based Motion Planning}\label{sec:relatedwork.batch}

BIT*~\cite{gammell_batch_informed_2015} approaches sampling-based motion planning by viewing batches of valid samples as vertices in a series of edge-implicit \emph{random geometric graphs} (RGGs).
This perspective allows planners to order their search in a principled manner and incorporate problem-specific information, such as cost heuristics.
AIT*~\cite{strub_aitstar_eitstar_2021} builds on this idea with an asymmetric bidirectional search.
The reverse search computes an accurate cost heuristic without edge validation to guide the forward search for a solution, focusing the forward search and reducing the number of fully evaluated edges.

TMIT* extends AIT* to plan in multimodal spaces.
AIT*'s lazy reverse search allows TMIT* to check the geometric feasibility of actions in a candidate symbolic plan before committing to computing its full motion plan.
\Citet{toussaint_logic-geometric_programming_2015} uses a similar hierarchy of feasibility checks to filter the space of symbolic plans.

\subsection{Symbolic Planning for TMP}\label{sec:relatedwork.symbolic}

TMP solvers adapt advances in standalone symbolic planning to the TMP context~\cite{garrett_pddlstream_integrating_2020,dantam_incremental_constraint-based_2018,srivastava_combined_task_2014,garrett_ffrob_leveraging_2017}.
Symbolic planning for TMP is uniquely challenging since valid symbolic plans may not correspond to valid motion plans.
TMP symbolic planners must be able to efficiently incorporate geometric feasibility constraints to generate alternative plans~\cite{dantam_incremental_constraint-based_2018}.
TMP solvers have addressed this need with custom action formulations~\cite{garrett_pddlstream_integrating_2020,srivastava_combined_task_2014}, top-$k$ symbolic planning with Monte Carlo tree search~\cite{ren_extended_task_2021}, and incremental constraint construction~\cite{dantam_incremental_constraint-based_2018}.

TMIT* extends the incremental Boolean-satisfiability-based (SAT-based) symbolic planner proposed by~\citet{kautz_planning_satisfiability_1992} and adapted for TMP by~\citet{dantam_incremental_constraint-based_2018}.
Compared to this original adaptation, we neither include symbolic representations of continuous scenegraphs nor prediscretize the continuous problem state.
We also solve a relaxed symbolic planning problem, do not require a translation step between discrete and continuous state, and use a new feature of the Z3 SMT solver~\cite{demoura_z3_efficient_2008} to improve planner performance and extensibility.\squeezeWords{}
\new{Other work has encoded the full Planning Domain Definition Language+ (PDDL+) symbolic language into SMT~\cite{cashmore_compilation_full_2016}.
TMIT* uses a subset of the simpler, more common PDDL 2.1.}

\subsection{Deferred Action Parameter Sampling}\label{sec:relatedwork.deferred}

Finding geometrically valid values for continuous action parameters (e.g., grasps and placement poses) is one of the core challenges of TMP~\cite{garrett_integrated_task_2020}.
Recent work~\cite{garrett_pddlstream_integrating_2020,migimatsu_object-centric_task_2019} attempts to reduce the computational burden of finding valid parameters by \new{validating actions incrementally}.

TMIT* adopts the differentiable distance function predicate representation of~\cite{thomason_unified_sampling-based_2019}.
This representation allows us to directly sample continuous states satisfying symbolic action preconditions and guide the motion planner toward these states.
This guidance lets TMIT* defer sampling action parameters until it has evidence from the motion planner that an action is both geometrically feasible and likely to be part of a valid solution.\looseness=-1

\section{Problem formulation}\label{sec:approach.problem}

A robot, $\robot$, comprises a kinematic tree of joint-connected links, $l \in \mathcal{L}$, and a mobile base with pose $P(\robot) \in \SE{2}$\footnote{TMIT* is not limited to planar mobile robots.}.
An object, $o \in \objects$, is a physical entity in the robot's environment with an associated 3D geometry and pose, $P(o) \in \SE{3}$.

A predicate defines a property of the combined robot and object configuration (\cref{def:predicate}) as a relation.
A symbol is a predicate applied to specific arguments (\cref{def:symbol}). 

\begin{definition}[Predicate]\label{def:predicate}
	A predicate, $p \in \predicates$, \new{is a function from a set of objects and/or robot links, $p_{\objects \cup \mathcal{L}} \subseteq \objects \cup \mathcal{L}$, to a Boolean domain, defining a problem-dependent property.
		Formally, $p: \prod_{n_p} p_{\objects \cup \mathcal{L}} \to \mathbb{B}$, where $n_p$ is the arity of $p$, and $\prod$ is the Cartesian product.}
	\new{Predicates may define properties with \emph{geometric} or \emph{discrete} interpretations (\cref{sec:approach.predicates})}.
\end{definition}

\begin{definition}[Symbol]\label{def:symbol}
	A symbol, $\sym \in \symbols$, is the application of a predicate, $p$, to specific arguments (also known as a proposition or ground predicate).
	A symbol is \true{} in a state if the resulting values of its arguments satisfy the associated predicate's property, and is \false{} otherwise\footnote{TMIT* extends to other SMT-encodable discrete symbolic domains.}. 
	Some symbols correspond to differences in the free configuration space, e.g., a symbol representing the robot's manipulator holding a specific object.
\end{definition}

The state space for our problem is the combination of the discrete and continuously valued dimensions (\cref{def:statespace}).
A mode is a subset of the state space with a unique, fixed setting of the discrete state and poses of ungrasped movable objects, and a \emph{mode family} refers to an infinite set of modes which share their discrete state setting (\cref{def:modes})~\cite{kingston_informing_multi-modal_2020}.
A motion is a continuous path through the state space (\cref{def:motion}). 
\begin{definition}[State Space]\label{def:statespace}
	The state space of the planning problem is the Cartesian product of the continuous robot configuration and object poses and the discrete symbolic state,
	\begin{equation}
		\statespace = \config_\robot \times \config_\objects \times \mathbb{B}^{\left|S\right|},
	\end{equation}
	where $\config_\robot$ is the robot configuration space~\cite{lavalle_planning_algorithms_2006}, $\config_\objects$ is the space of all the objects' poses, and $\mathbb{B}^{\left|S\right|}$ is a Boolean domain representing the value of the \new{discrete} symbols.
\end{definition}

\begin{definition}[Modes and Mode Families]\label{def:modes}
	A mode family, $\ModeFam \subseteq \statespace$, of a \new{specific} setting of the discrete state, $\xi \in \BoolDom^{|S|}$, is defined as:
	\begin{equation}
		\ModeFam = \left\{(\state_\robot \in \config_\robot, \state_\objects \in \config_\objects, \xi)\right\}
	\end{equation}
	A mode, $\mode \subseteq \ModeFam$, for a \new{specific} set of movable object poses, $\rho \in \config_\objects$, is defined as:
	\begin{equation}
		\mode = \left\{(\state_\robot \in \config_\robot, \rho, \xi)\right\}
	\end{equation}
\end{definition}
\begin{definition}[Motion]\label{def:motion}
	A motion between two states, $\state_i, \state_f \in \statespace$, is a sequence of states described by a function, $\motion: \left[0, 1\right] \to \statespace$, that starts at the initial state, $\motion\left(0\right) = \state_i$, ends at the final state, $\motion\left(1\right) = \state_f$, and is continuous at all points, $\forall t, \lim_{\tau \to t} \motion(\tau) = \motion(t)$.
	By definition, a motion cannot change the discrete symbolic state.
\end{definition}

Discrete actions are defined by their necessary preconditions and their resulting effects (\cref{def:constraint,def:effect,def:action}).

\begin{definition}[Precondition]\label{def:constraint}
	A precondition or constraint,
	$
		\constraint: \statespace \to \mathbb{B}
	$,
	is a function that evaluates a Boolean propositional logic formula at a state.
\end{definition}

\begin{definition}[Effect]\label{def:effect}
	An effect,
	$
		\effect: \statespace \to \statespace
	$,
	is a deterministic function that possibly modifies a state in a discontinuous manner.
	An effect may alter the continuous and discrete parts of a state, but many TMP action effects only modify the discrete part.
\end{definition}

\begin{definition}[Action]\label{def:action}
	An action, $\action \in \actions$, $\action = \left(\constraint_\action, \effect_\action\right)$, comprises the preconditions, $\constraint_\action$, necessary to execute the action and the effects, $\effect_\action$, of the action on the state when successfully executed, as in symbolic planning~\cite{ghallab_automated_planning_2014}.
	The set of actions, $\actions$, includes the \emph{null action}, $\bot$, defined such that,
	$\forall \state \in \statespace,\, \constraint_\bot\left(\state\right) = \true$, and,
	$\forall \state \in \statespace,\, \effect_\bot\left(\state\right) = \state$.
	An action can be considered an abstraction of a specific high-level robot skill (e.g., grasping, pushing, pouring, etc.).
\end{definition}

The TMP problem is then formally defined as the search for a plan of actions and motions (\cref{def:problem}).

\begin{definition}[Task and Motion Planning (TMP) Problem]\label{def:problem}
	Let $\state_0 \in \statespace$ be an initial state, $\constraint_\text{g}$ be a goal specified as a constraint, and $\actions$ be a set of discrete actions executable by a robot.
	The TMP problem is then formally defined as the search for motions, $\motion_i \in \motions$, and symbolic actions, $\action_i \in \actions$, that can be interleaved into a task-and-motion plan, $(\motion_1, \action_1, \motion_2, \action_2, \ldots, \motion_n, \action_n)$, such that:
	\begin{enumerate}
		\item The plan begins at the initial state, $\sigma_1(0) = \state_0$.
		\item Robot and object motions are valid (e.g., collision free),
		      \begin{equation*}
			      \forall i = 1, 2, \ldots, n,\, \forall t \in \left[0, 1\right], \mathtt{is\_valid}\left(\motion_i(t)\right) = \true{}.
		      \end{equation*}
		\item The final state of each motion, $\motion_i(1)$, satisfies the requirements to execute the following action, $\action_{i}$,
		      \begin{equation*}
			      \forall i = 1, 2, \ldots, n,\; \constraint_{\action_{i}}(\motion_i(1)) = \true{}.
		      \end{equation*}
		\item The effect of each intermediate action, $\action_{i}$, results in the initial state of the following motion, $\motion_{i+1}(0)$, and
		      \begin{equation*}
			      \forall i = 1, 2, \ldots, n-1,\; \effect_{\action_i}(\motion_{i}(1)) = \motion_{i+1}(0).
		      \end{equation*}
		\item The effect of the final action, $\effect_{\action_n}$, meets the specified goal constraint,
		      $\constraint_\text{g}(\effect_{\action_n}(\motion_n(1))) = \true{}$.
	\end{enumerate}
\end{definition}

Optimal TMP finds TMP solutions which optimize a given cost function (\cref{def:optimal.tmp}).

\begin{definition}[Optimal TMP]\label{def:optimal.tmp}
	Let $\Theta$ be the set of all valid solutions to a TMP problem (\cref{def:problem}) and $\cost: \Theta \to \mathbb{R}^{\geq 0}$ be a cost function.
	The optimal TMP problem is the search for \new{a} lowest cost solution, $\theta^* \in \Theta$, such that
	$
		\theta^* = \argmin_{\theta \in \Theta}{\cost(\theta)}
	$
\end{definition}

Solving TMP as independent symbolic and motion planning problems is commonly infeasible~\cite{garrett_integrated_task_2020}; TMIT* presents a holistic approach that solves the integrated problem.


\section{Task and Motion Informed Trees (TMIT*)}\label{sec:approach}

TMIT* (\cref{fig:planning-loop}) plans in a multimodal hybrid state space \cite{thomason_unified_sampling-based_2019} in which any valid path, annotated with symbolic actions at certain states, is a valid task-and-motion plan.
It uses an incremental SMT-based symbolic planner inspired by~\citet{dantam_incremental_constraint-based_2018} on a relaxed problem to find a candidate symbolic plan, defining a sequence of hybrid state space modes.
It samples batches of states along this mode sequence and uses a distance-based geometric predicate representation~\cite{thomason_unified_sampling-based_2019} to detect when samples are within the connection radius~\cite{karaman_sampling-based_algorithms_2011} of precondition-satisfying states for an action in the plan.
Sampling these precondition-satisfying states corresponds to choosing continuous action parameters (e.g., grasps) and allows TMIT* to connect to the next reachable modes.

If this marched sampling procedure does not reach the goal mode then the symbolic plan is infeasible at the current resolution and TMIT* uses knowledge of the modes in which it failed to inform the search for a new symbolic plan.
Sampling resumes without discarding old symbolic plans or samples to reuse motion planning effort and allow increased sample resolution to prove symbolic plan feasibility.
If the sampling procedure does reach the goal mode then the symbolic plan is validated further via an asymmetric bidirectional search~\cite{strub_aitstar_eitstar_2021}.
This search process continues until a solution is found and almost-surely converges asymptotically towards the jointly optimal task-and-motion solution.


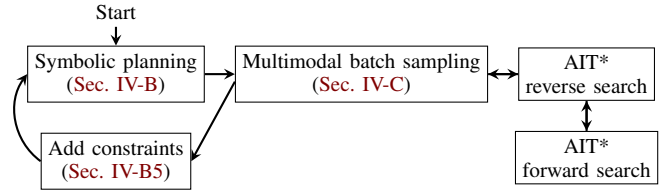
\begin{figure}[t]
	\centering
  \footnotesize
	\begin{tikzpicture}[node distance = 0.4cm, auto]
		\node (symbolic) [process] {Symbolic planning\\(\cref{sec:approach.symbolic})};
    \node (start) [above=0.25cm of symbolic] {Start};
		\node (batch) [process, right=of symbolic] {Multimodal batch sampling\\(\cref{sec:approach.sampling})};
		\node (constraints) [process, below=of symbolic] {Add constraints\\(\cref{sec:alternative.plans})};
		\node (reverse) [process, right=of batch] {AIT*\\reverse search};
		\node (forward) [process, below=of reverse] {AIT*\\forward search};

    \draw [arrow] (start) -- (symbolic);
    \draw [arrow] (symbolic) -- (batch);
    \draw [arrow] (batch.east) -- (reverse.west);
    \draw [arrow] (reverse) -- (forward);
    \draw [arrow] (forward) -- (reverse);
    \draw [arrow] (constraints.west) to [bend left=45] (symbolic.west);
    \draw [arrow] (reverse.west) -- (batch.east);
    \draw [arrow] ([yshift=-0.1cm]batch.west) -- (constraints.east);
	\end{tikzpicture}%
  \caption{The TMIT* planning loop. TMIT* cycles between generating candidate symbolic plans, (\cref{sec:approach.symbolic}), sampling batches of states in the modes traversed by the candidate plans (\cref{sec:approach.sampling}), and the reverse and forward search stages of AIT*~\cite{strub_aitstar_eitstar_2021}. If it fails to reach the goal mode, it adds constraints (\cref{sec:alternative.plans}) and returns to the symbolic planner for a new candidate. This process continues until it finds a valid path to the goal.}\label{fig:planning-loop}
\end{figure}

\subsection{Predicate representation}\label{sec:approach.predicates}

TMIT* uses a split predicate representation to simplify the symbolic planning subproblem and guide the motion plan search toward action-precondition-satisfying states.

We partition the set of predicates into those with \emph{geometric} properties, $\geometricpreds \subseteq \predicates$, and those with \emph{discrete} properties, $\discretepreds = \predicates \setminus \geometricpreds$.
Some predicates have natural categorizations, e.g., a discrete predicate representing a light's state or a geometric predicate describing distance to an object.
For others, the correct interpretation may be problem specific, e.g., a predicate for an object being on a specific surface.
\new{This distinction aids in representing and solving TMP problems.}

Discrete predicates comprise the discrete part of the state space (\cref{def:statespace}).
Geometric predicates have associated differentiable functions defining the distance from a given state to \new{a} nearest predicate-satisfying state.
The distance for a geometric predicate, $p \in \geometricpreds$, is zero at a given state, $\state \in \statespace$, if and only if the state, $\state$, satisfies the predicate, $p$.
We project uniform-randomly sampled states onto the zero-level set of this function via gradient-based optimization~\cite{nocedal_numerical_optimization_2006}.

Predicate distance functions can be defined in different ways~\cite{osher_signed_distance_2003, thomason_counterexample-guided_repair_2021,thomason_unified_sampling-based_2019}.
We automatically derive distance functions for symbolic formulae with ``unsatisfaction semantics''~\cite{thomason_unified_sampling-based_2019}.

\subsection{Symbolic planning}\label{sec:approach.symbolic}

TMIT* solves the symbolic planning subproblem with an incremental SMT-based symbolic planning algorithm inspired by~\citet{dantam_incremental_constraint-based_2018}.
SMT solvers generalize SAT by extending the variable types (to include integers, real numbers, arrays, etc.) and relations, which are backed by efficient solvers for the corresponding formal theories~\cite{biere_handbook_satisfiability_2009}.

Basic SAT-based symbolic planning~\cite{kautz_encoding_plans_1996}
\begin{enumerate*}
	\item creates Boolean variables for each action and symbol,
	\item adds constraints encoding the initial and goal state,
	\item constrains variables to remain consistent between steps (i.e., the ``frame axioms''),
	\item constrains actions to imply their preconditions and effects, and
	\item iteratively increases the number of plan steps until the resulting formula is satisfiable
\end{enumerate*}.
\Citet{dantam_incremental_constraint-based_2018} extend this core approach by encoding discretized geometric state and using the Z3 SMT solver's~\cite{demoura_z3_efficient_2008} incremental constraint stack to reuse solver effort when adding plan steps.

We further extend the approach of~\citet{dantam_incremental_constraint-based_2018} with a custom theory that improves performance and expressivity.
We specificially
\begin{enumerate*}
	\item enforce the frame axioms, precondition constraints, and other planning constraints implicitly via Z3's \emph{user propagator}, reducing the number of long-lived clauses,
	\item relax the symbolic planning problem by omitting geometric propositions from precondition formulae, and
	\item do not discretize the continuous state.
\end{enumerate*}
We outline TMIT*'s symbolic planner in~\cref{sec:encoding.state,sec:encoding.actions,sec:encoding.frame.and.mutex,sec:symbolic.loop,sec:alternative.plans}.

\subsubsection{State encoding}\label{sec:encoding.state}

In contrast to~\cite{dantam_incremental_constraint-based_2018}, we create Boolean indicator variables only for each \emph{discrete} symbol (i.e., those corresponding to predicates in \discretepreds) at each step of the plan and omit the geometric symbols.
We also do not discretize or symbolically represent the geometric state in the symbolic planning problem, \new{relaxing the symbolic planning problem by optimistically ignoring constraints from geometric predicates.
This frees the SMT solver from explicitly considering geometric relations.\squeezeWords}

\subsubsection{Action encoding}\label{sec:encoding.actions}

We create a Boolean indicator variable, \new{$\bvar{\action}^j \in \mathbb{B}$, for each action, $\action$}, at each step $j$.
This indicator variable is \texttt{true} if the plan takes the associated action at step $j$ and \texttt{false} otherwise.
Our custom theory implicitly enforces the precondition and effect constraints, \new{$\bvar{\action}^j \implies \constraint^{j-1}_{\action} \wedge \effect^j_{\action}$}, where \new{$\constraint^{j-1}_{\action}$} is the precondition of $\action$ encoded using the variables for the relevant \new{discrete} symbols at step $j - 1$, and \new{$\effect^j_{\action}$} is the effect of $\action$ applied at step $j$.
When Z3 assigns a value to an action indicator variable, we assert any unsatisfied clauses in the relevant precondition and effect constraints.
This improves performance by reducing the number of clauses for the SMT solver to validate.\squeezeWords

\subsubsection{Frame axioms and action mutexes}\label{sec:encoding.frame.and.mutex}

We also implicitly enforce the frame axioms and action mutex constraints.
The frame axioms specify that a symbol's value at step $j$ is the same as at the preceding step unless an action that modifies it is chosen at step $j$.
This ensures that variable values remain consistent between steps.
When Z3 assigns a symbol indicator variable's value at a step, we assert the relevant frame axiom.
\new{Action mutex constraints force Z3 to choose only a single action per step.}\squeezeWords

\subsubsection{The symbolic planning loop}\label{sec:symbolic.loop}

The symbolic planning loop is a typical incremental SAT-based planner~\cite{dantam_incremental_constraint-based_2018}.
We first assert the initial discrete state and add plan steps up to a heuristically determined minimum plan length\footnote{Zero, unless we have a better estimate for a problem.}.
Each step adds constraints asserting the discrete state at the step and the action transition constraints and frame axioms (\cref{sec:encoding.actions,sec:encoding.frame.and.mutex}).
We then assert the goal constraint and invoke the Z3 SMT solver~\cite{demoura_z3_efficient_2008} to attempt to find a solution.
We add new steps and reinvoke the solver until a satisfying assignment exists.
We extract a plan from a solution by checking which action indicator variables are \texttt{true} at each step.

\subsubsection{Generating alternative plan candidates}\label{sec:alternative.plans}

TMIT* has two methods of generating alternative symbolic plan candidates.
\new{The first enumerates all candidate symbolic plans by requiring new plans to differ from previous plans by at least one action, ensuring completeness~\cite{dantam_incremental_constraint-based_2018}.}
This method adds constraints:
\begin{equation}
	\neg \bigwedge_{1 \le j \le n} \left\lparen \bigwedge_{\action \in \actions} \bvar{\action}^j = \new{\texttt{value}_k}(\bvar{\action}^j)\right\rparen\label{eqn:alternative}
\end{equation}
for a current solution with length $n$, and where \new{$\texttt{value}_k$} returns the value of a variable in the \new{$k$-th candidate plan}.

A \emph{prefix} of a symbolic plan is a subsequence of the actions in the plan starting with the initial action.
The second method forces new plans to avoid \emph{failing prefixes} of symbolic plan candidates by adding constraints of the same form as~\cref{eqn:alternative}, but only until the first step for which precondition-satisfying state \new{sampling failed}.
\new{Precondition-satisfying state sampling may fail if sample projection (\cref{sec:approach.predicates}) fails to converge, if it converges to a local minimum off the precondition-satisfying manifold, or if the resulting state is invalid (i.e., in collision).}
This can quickly eliminate broader groups of candidate symbolic plans and find a solution more efficiently, but may remove prefixes that \new{would prove feasible with more computation}.
We use this prefix-blocking method in the experiments of~\cref{sec:evaluation}.\squeezeWords

\subsection{Motion Planner Integration}\label{sec:approach.sampling}

TMIT* solves the motion planning subproblem by building upon AIT*~\cite{strub_aitstar_eitstar_2021}, an almost-surely asymptotically optimal batch-sampling-based motion planner, to find geometrically valid instantiations of candidate symbolic plans.
We extend AIT*'s batch sampling to
\begin{enumerate*}
	\item sample states in each reachable mode and
  \item sample action precondition-satisfying states only if a uniform sample is within \new{a tunable distance threshold (e.g., the connection radius~\cite{karaman_sampling-based_algorithms_2011})} of the precondition-satisfying region\squeezeWords
\end{enumerate*}.\squeezeWords

A mode is \emph{reachable} if either it is the initial mode (i.e., the initial discrete state and scene) or we have sampled a precondition-satisfying state for an action that transitions to it from a reachable mode.
Sampling batches across the reachable modes uniformly increases the resolution of AIT*'s RGG \new{(\cref{sec:relatedwork.batch})} and allows TMIT* to implicitly reevaluate old candidate symbolic plans without backtracking.

The set of precondition-satisfying states for an action is often a manifold of measure zero in the ambient configuration space due to dimensionality-reducing constraints and therefore has zero probability of being sampled with uniform-random sampling.
Computing precondition-satisfying states is computationally expensive relative to uniform-random sampling\new{, and many such states are challenging to reach with a motion plan (e.g., states close to objects being manipulated).
}

We \new{avoid wasted effort by} projecting uniform-random samples onto precondition-satisfying manifolds \new{only} when the samples are within \new{a tunable distance threshold (e.g.,~\cite{karaman_sampling-based_algorithms_2011})}
of the manifold\footnote{Sampling a precondition-satisfying state takes on the order of 10--100 optimizer iterations; testing the distance takes less than one~\cite{thomason_unified_sampling-based_2019}.}, which effectively \emph{inflates} the manifold to positive measure.
This strategy avoids computing unusable precondition-satisfying samples by only invoking this process starting from states that are
\begin{enumerate*}
	\item in the RGG and
	\item close enough to a precondition region to improve the solution cost
\end{enumerate*}.
Starting from valid states close to a precondition region may additionally improve the likelihood that the resulting precondition-satisfying sample will be valid.
\newlength{\textfloatsepsave}%
\setlength{\textfloatsepsave}{\textfloatsep}%
\setlength{\textfloatsep}{0pt}%
\begin{algorithm}[t]
  \small
	\SetAlgoLined{}
	\DontPrintSemicolon{}
	\SetKwFunction{Pop}{pop}
	\SetKwFunction{SampleValid}{SampleValid}
	\SetKwFunction{Valid}{IsValid}
	\SetKwFunction{NeedSamples}{NeedSamples}
	\SetKwFunction{SampleConstraint}{SamplePrecond}
	\SetKwFunction{ShouldAttempt}{Viable}
	\SetKwFunction{UpdateReachableModes}{UpdateModes}
	\SetKwFunction{ReachedMode}{AtGoal}
	\SetKwFunction{NoActions}{NoActions}
	\SetKwFunction{UpdateTaskPlan}{NewTaskPlan}
	\SetKwFunction{IncreaseThreshold}{IncreaseBudget}
	\SetStartEndCondition{ }{}{}%
	\SetKw{KwTo}{in}
	\SetKwFor{For}{for}{\string:}{}%
	\SetKwIF{If}{ElseIf}{Else}{if}{:}{elif}{else:}{}%
	\SetKwFor{While}{while}{:}{fintq}%
	\SetKw{And}{and}
	\SetKw{Not}{not}
	\SetKw{Or}{or}
	\SetInd{0.2em}{0.5em}
	\KwIn{Mode queue $\omega$, connection radius $\mu$, goal $\constraint_g$}
	\KwOut{Batch of samples $B$}
	$B \gets \{\}$\;
	\While(\tcp*[f]{All reachable modes}){$|\omega| > 0$}{
		$\mode \gets$ \Pop{$\omega$}\;
		\While{\NeedSamples{B}}{
			$s \gets \SampleValid{\mode}$\;
			$B \gets B \cup \{s\}$\;
			\lIf{$\constraint_g(s)$}{\Return{} $B$}
			\For{$\action_i \in \mode.\texttt{viable\_actions}$}{\label{lst:viable}
				\If{$d(\constraint_i, s) < \mu$}{
					$s' \gets$ \SampleConstraint{$\constraint_i$, s}\;
					\If{\Valid{$s'$}}{
						$B \gets B \cup \{s'\}$\;
						\UpdateReachableModes{$\omega$, $s'$}\;
					}
				}
			}
		}
	}
	\lIf{\NoActions{}}{\IncreaseThreshold{}}
	\lIf{\Not{} \ReachedMode{} \Or{} \NoActions{}}{\UpdateTaskPlan{}}
	\Return{$B$}
	\caption{Multimodal batch sampling}\label{alg:batch}
\end{algorithm}
\setlength{\textfloatsep}{\textfloatsepsave}%
\Cref{alg:batch} shows the multimodal batch sampling function.
\SampleValid\ draws each state uniformly at random from the valid configuration space of a given mode.
The viable actions (\cref{lst:viable}) of a mode, $\mode \in \ModeFam$, are symbolic actions that (1) are used at \ModeFam{} in a candidate symbolic plan and (2) have been attempted less than a heuristically determined number of times.
Attempting an action means trying to sample a valid state satisfying its precondition constraint.
The heuristic limit on attempts per action provides a budget of computation per candidate symbolic plan; \IncreaseThreshold\ increments this heuristic threshold.
\NoActions\ tests if any actions in any reachable mode are viable, and $d(\cdot, \cdot)$ returns the distance from a state to the nearest precondition-satisfying state.
\SampleConstraint\ projects the given state onto the manifold of precondition-satisfying states by gradient-based optimization.
\UpdateReachableModes\ adds newly reached modes to the mode queue, \ReachedMode\ checks if the total set of samples contains goal mode states, and \UpdateTaskPlan\ invokes the task planner (\cref{sec:symbolic.loop,sec:alternative.plans}).

The mode queue is ordered to prioritize recently reached modes.
This ordering creates behavior akin to the ``enforced hill climbing'' of the FastForward (FF) task planner~\cite{hoffmann_ff_planning_2001} by continuing the search in the resulting mode when an action succeeds, effectively following the corresponding task plan candidate as far as possible.
\Cref{alg:batch} returns early if it reaches a goal-satisfying state.
The mode queue persists across invocations of~\cref{alg:batch} for the same batch.

\subsection{Considerations for Multimodal AIT*}

Computing exact distance in the hybrid configuration space is PSPACE-complete\footnote{The distance between two states is a function of the shortest sequence of mode transitions between them, which is an optimal symbolic plan.}~\cite{vega-brown_task_motion_2020}.
We conservatively over-approximate configuration space distance by assuming states in different modes are infinitely far apart if we do not have a successful transition between their modes.
This approximation is not a metric, but suffices in practice.

AIT* uses the reverse search to calculate a heuristic to guide the forward search \new{(see~\cite{strub_aitstar_eitstar_2021} for details).
	This approximation must account for the conservative over-approximation of intermode distance.}
The reverse search therefore uses forward-direction transitions and distances between modes (distances are directionally symmetric within a mode).
%
%

\subsection{Implementation}

We provide a proof-of-concept implementation of TMIT* in C++\footnote{\url{https://robotic-esp.com/code/tmitstar}}.
We use the implementation of AIT* and associated sampling-based motion planning utilities from the Open Motion Planning Library~\cite{sucan_open_motion_2012} and use Bullet~\cite{coumans_bullet_physics_2013} for collision checking.
The geometric predicate implementation is an improved version of~\cite{thomason_unified_sampling-based_2019} using the Autodiff~\cite{leal_autodiff_modern_2018} library for automatic differentiation, NLOpt~\cite{johnson_nlopt_nonlinear-optimization} for optimization, and a bespoke dual-number automatic differentiation implementation in LuaJIT for predicate functions.

The planner's input is simpler than most other TMP solvers and does not include specialized samplers or planners, or prediscretized state.
It requires only a description of the initial scene, a specification of the robot morphology and kinematics, object and robot geometries, the symbolic planning domain and problem, and functions for geometric predicates.
%
%

\section{Analysis}\label{sec:analysis}

The probabilistic completeness and almost-sure asymptotic optimality of TMIT* follow from the corresponding properties of AIT*.
We sketch proofs of these properties for TMIT*.
In the following, assume that precondition regions are convex\new{, and that all precondition-satisfying states in a complete feasible plan are surrounded by a valid hyperball~\cite{shome_pushing_boundaries_2020}}.

\begin{theorem}[Probabilistic Completeness]
	Given a TMP problem as in~\cref{def:problem}, the probability that TMIT* does not find a solution goes to zero as the number of samples taken approaches infinity.
\end{theorem}

\begin{proof}
	TMIT* eventually attempts every possible symbolic plan candidate, since SATPlan (\cref{sec:approach.symbolic}) is complete~\cite{kautz_encoding_plans_1996}, and the constraints described in~\cref{sec:alternative.plans} ensure that candidate symbolic plans are distinct.
	The symbolic planner will be invoked infinitely often until a TMP solution is found, since \Cref{alg:batch} requests a new task plan whenever it fails to reach the goal mode or runs out of viable actions.
	Actions have a finite attempt budget and the mode queue is always emptied before completing a batch.
	AIT* is probabilistically complete~\cite{strub_aitstar_eitstar_2021} and will eventually sample within the connection radius of each precondition region infinitely often.
	Each precondition-satisfying manifold is convex by assumption, so projecting uniform-random samples onto the precondition regions will eventually sample every point in the manifold.
	Thus, we will find a valid path through our planning space if one exists, and therefore are probabilistically complete.
\end{proof}

The sketch of almost-sure asymptotic optimality is similar.
\begin{theorem}[Almost-Sure Asymptotic Optimality]
	Given an optimal TMP problem as in~\cref{def:optimal.tmp}, TMIT* converges asymptotically to an optimal-cost solution with probability one as the number of samples approaches infinity.
\end{theorem}

\begin{proof}
	AIT* is asymptotically optimal within each mode and its reverse search heuristic is computed globally across modes.
	TMIT* samples precondition-satisfying states whenever a uniform-random sample is within the connection radius of the precondition region.
	TMIT*'s symbolic planner is complete, each precondition region is inflated to positive measure, and precondition regions are convex by assumption, so TMIT* will asymptotically sample better precondition-satisfying states for every viable action.
	Thus, as the number of uniform-random samples approaches infinity, we asymptotically converge to the optimal action sequence and associated optimal motion plan (the optimal task and motion plan) with probability one.
\end{proof}

\section{Evaluation}\label{sec:evaluation}
Directly comparing TMP solvers is challenging due to differing definitions of the TMP problem.
TMIT* does not assume prediscretization of continuous state~\cite{dantam_incremental_constraint-based_2018} or specialized blackbox precondition samplers~\cite{garrett_pddlstream_integrating_2020}, which makes its problems harder.
We perform a cold-data comparison to Planet~\cite{thomason_unified_sampling-based_2019}, as its problem definition and assumptions are closest to TMIT*'s.
All experiments were run on an AMD Ryzen 7 2700X CPU with 32 GB of RAM and use a PR2 robot model with 14 controllable joints and a planar mobile base \new{(17 total degrees of freedom)}.
We evaluate on two common TMP tasks: clutter clearing and shelf rearrangement.
We use batches of 50 samples per mode, a batch budget (the number of batches before requesting a new symbolic plan candidate) of five for clutter clearing and two for shelf rearrangement, and an initial action-attempt budget of one.
We use Euclidean distance within modes and optimize path length.
\squeezeWords

\paragraph{Clutter Clearing}

A set of colored sticks is initially scattered on one of two tables (\cref{fig:clutter.scene}).
The goal is to place each stick on one of two other tables corresponding to its color.
Every stick must be manipulated, and sticks occlude others in their initial positions, so solutions require reasoning over a long series of symbolic actions to move the sticks in a geometrically valid order.
The action space contains $16 \times |\objects|$ symbolic actions for each instance size and each symbolic action has an infinite number of possible continuous instantiations.
\Cref{fig:clutter.results} shows results for clutter clearing.

\paragraph{Shelf rearrangement}

The robot must retrieve and move a set of objects between two stacked shelf surfaces (\cref{fig:shelf.scene}).
The target objects are initially placed deep within the shelves and are surrounded by increasing numbers of distractor objects.
Motions are geometrically constrained by the shelves and the distractor objects must either be moved out of the way or maneuvered around to reach the target objects.
The action space contains $6 \times |\objects|$ symbolic actions for each instance size.
\Cref{fig:shelves.results} shows results for shelf rearrangement.

\subsection{Qualitative Discussion}

\Cref{fig:clutter.time} shows that TMIT* significantly outperforms Planet in initial solution time, often by an order of magnitude.
The results for Planet constitute 10 successful trials of each instance size; in contrast, the TMIT* results constitute 100 trials where timeouts are considered to have taken infinite time.
\Cref{fig:clutter.percent} shows that TMIT* finds initial solutions for most instances quickly but that larger problem instances are more likely to either time out or have higher variance in initial solution times.
This distribution reflects the greater geometric and symbolic challenge of the more complex instances.
\Cref{fig:clutter.optimizing} shows median solution costs for 100 trials of TMIT* on instances of clutter clearing with 3--5 target objects.
While the largest drop in median cost corresponds to initial solution discovery, the trends show that TMIT* makes consistent progress toward lower-cost solutions.

\Cref{fig:shelves.time} demonstrates the relative effect of minimum solution length versus the number of objects in a problem environment on TMIT*'s initial solution performance.
Although there are more objects present in large shelf rearrangement instances than large clutter clearing instances, the number of actions necessary to solve a shelf rearrangement problem is generally lower than the number required for a comparably large clutter clearing problem.
This results in TMIT* finding solutions for large shelf rearrangement instances faster than for large clutter clearing problems.
TMIT*'s motion-planner-guided sampling of continuous action parameters also sometimes allows it to find grasp poses for the target blocks that carefully reach past the distractor objects and reduce the overall plan length.
\Cref{fig:shelves.optimizing} shows TMIT* optimizing costs for instances of the shelf rearrangement problem.

\section{Conclusions}
\begin{figure*}[tp]
	\centering\subfloat[\label{fig:clutter.time}]{\includegraphics[width=0.33\textwidth]{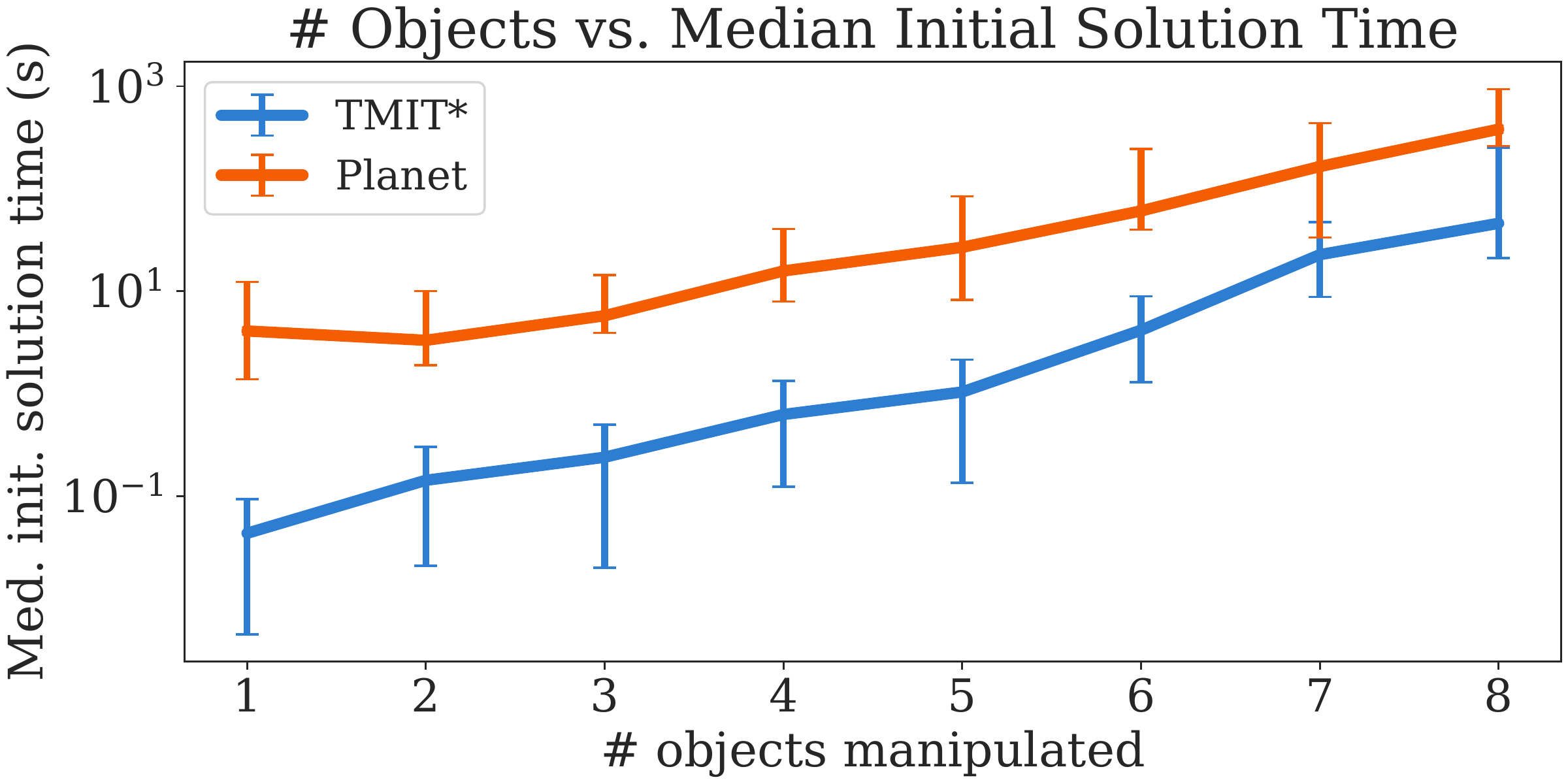}}%
	\hfill
	\subfloat[\label{fig:clutter.percent}]{\includegraphics[width=0.33\textwidth]{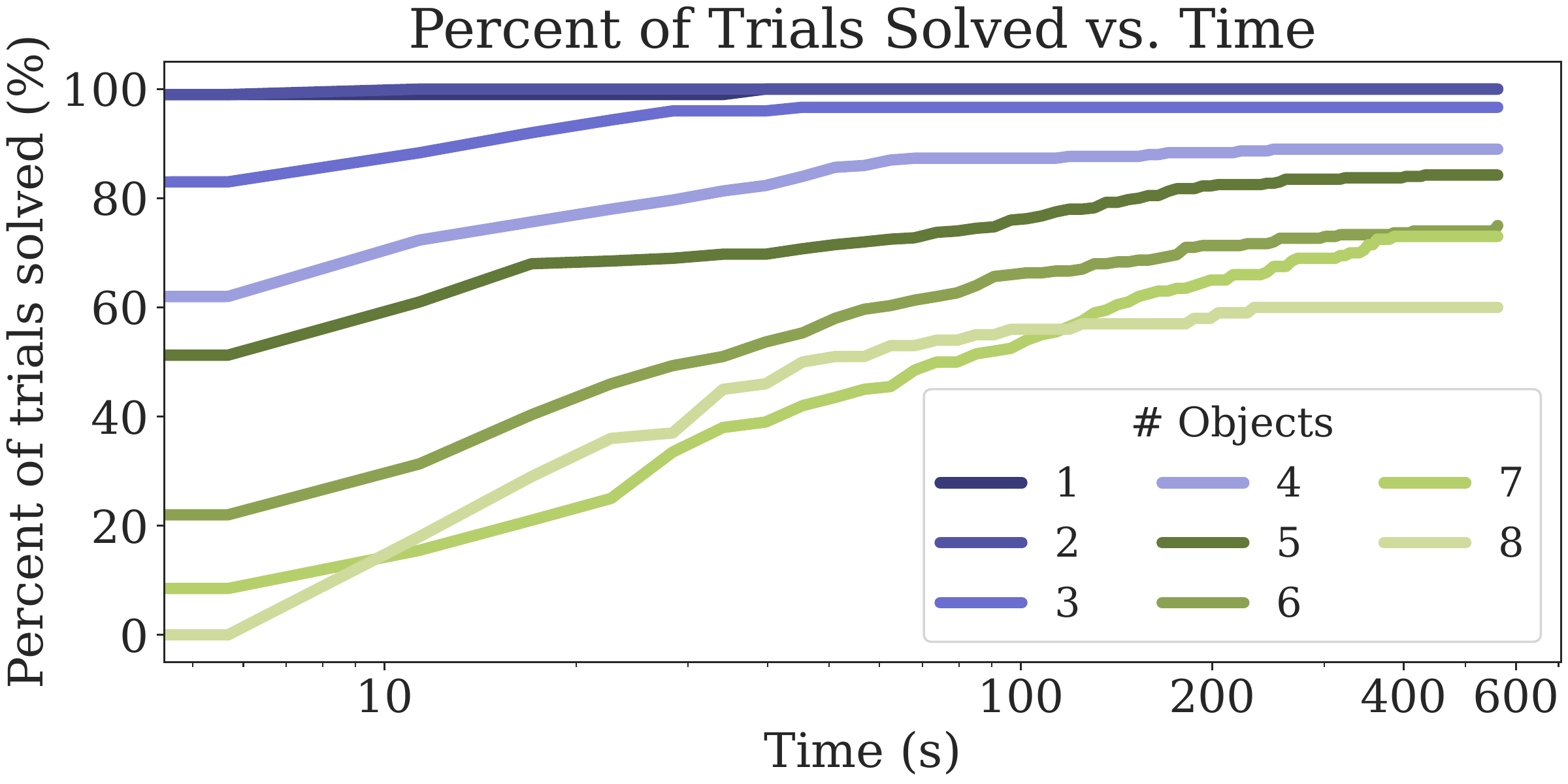}}%
	\hfill
	\subfloat[\label{fig:clutter.optimizing}]{\includegraphics[width=0.33\textwidth]{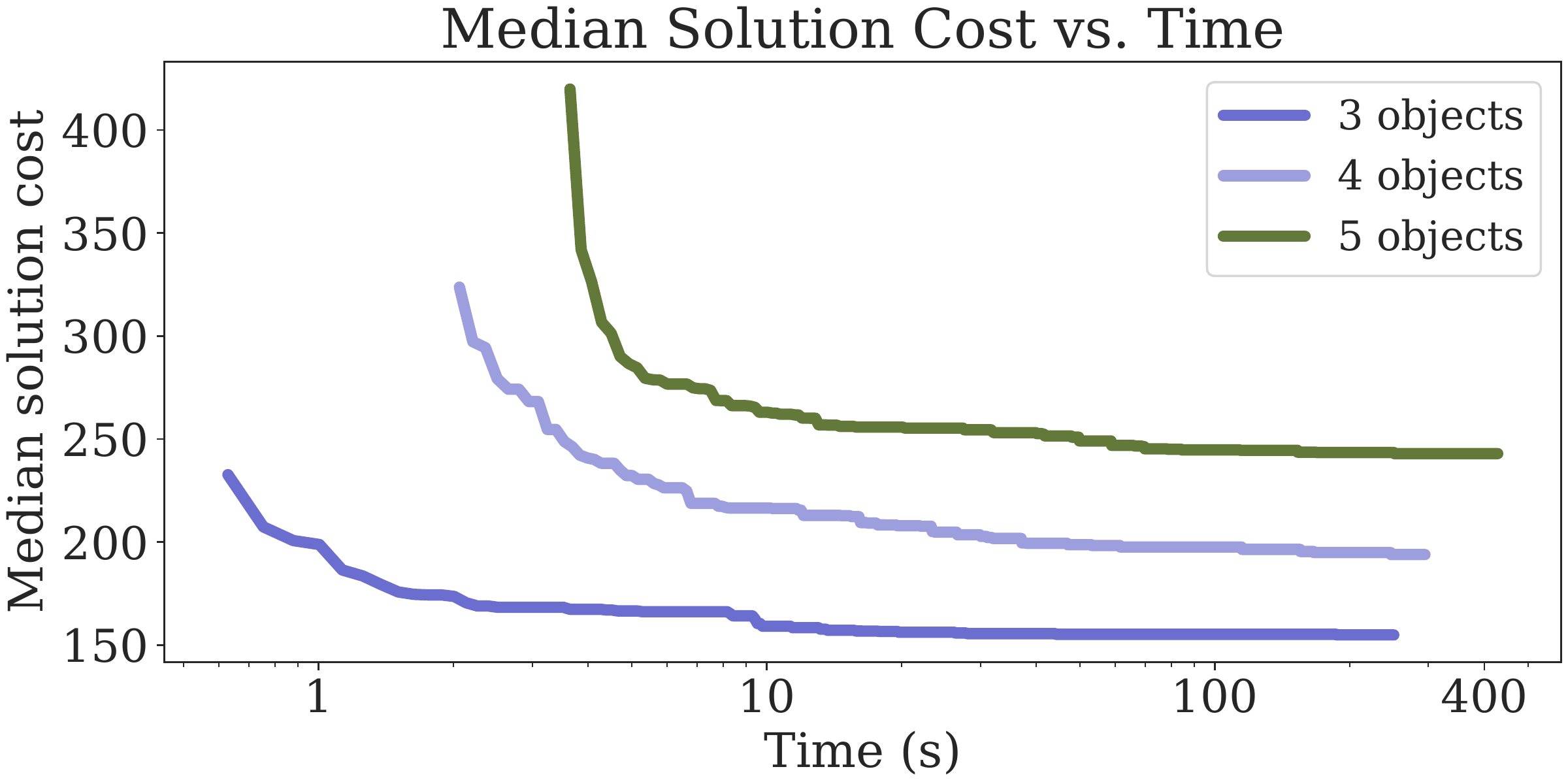}}%
	\caption{%
		Results for 100 trials of clutter clearing.
		\Cref{fig:clutter.time,fig:clutter.optimizing} show median time or cost (respectively); the error bars in~\cref{fig:clutter.time} show a 95\% confidence interval of the median.
		Each trial had 600 seconds of planning time; trials which failed to find a solution within this bound were counted as having infinite duration and cost.
		Trials in~\cref{fig:clutter.optimizing} terminated early if their cost converged.
		\Cref{fig:clutter.percent} show the cumulative percentage of problems solved as a function of time for each instance size.
		Timeouts are more frequent for larger instance sizes, reflecting the greater geometric difficulty of the problem.
	}\label{fig:clutter.results}
\end{figure*}
\begin{figure*}[tp]
	\centering\subfloat[\label{fig:shelves.time}]{\includegraphics[height=0.16\textwidth,width=0.33\textwidth]{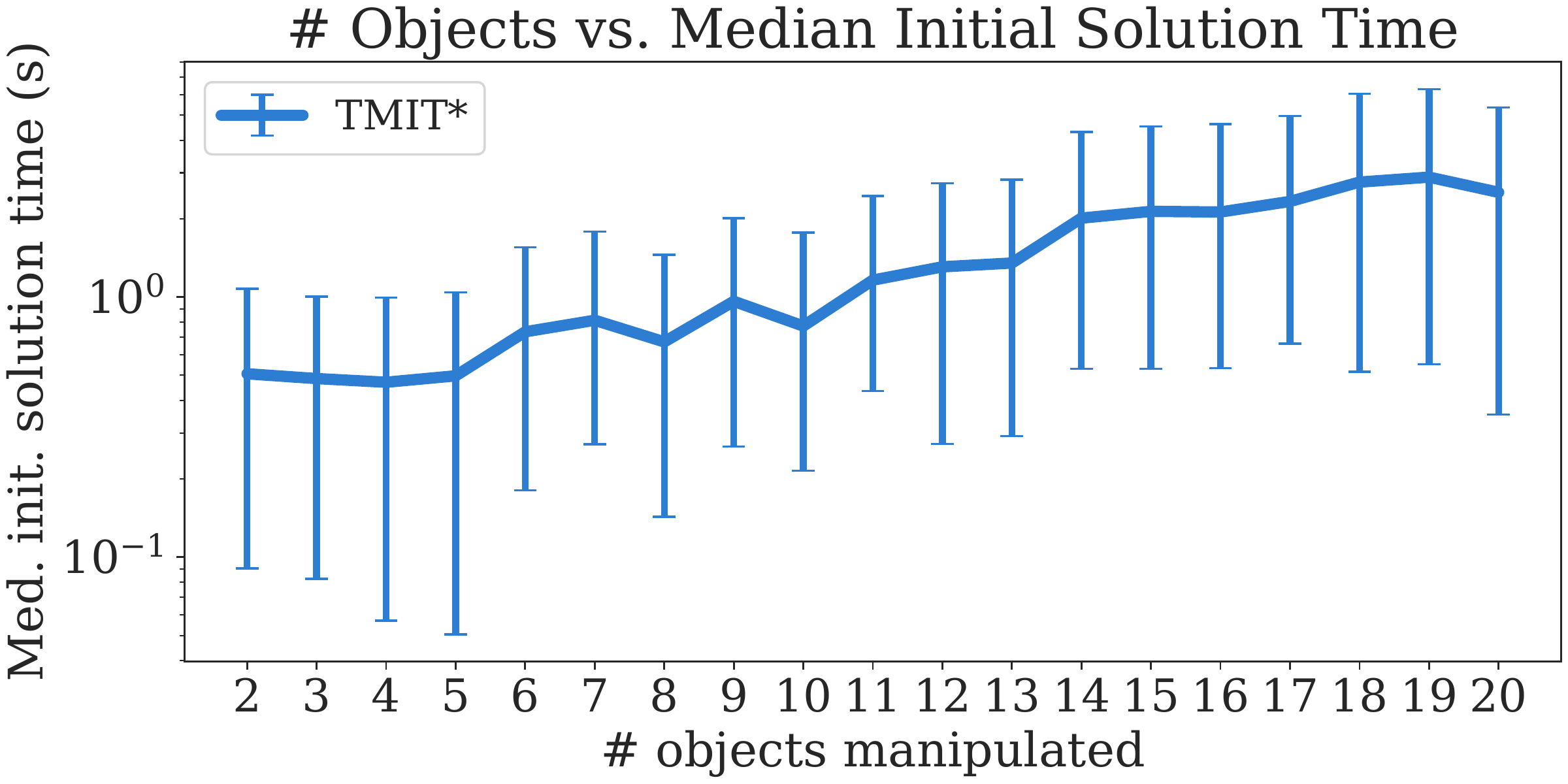}}%
	\hfill
	\subfloat[\label{fig:shelves.percent}]{\includegraphics[height=0.16\textwidth,width=0.33\textwidth]{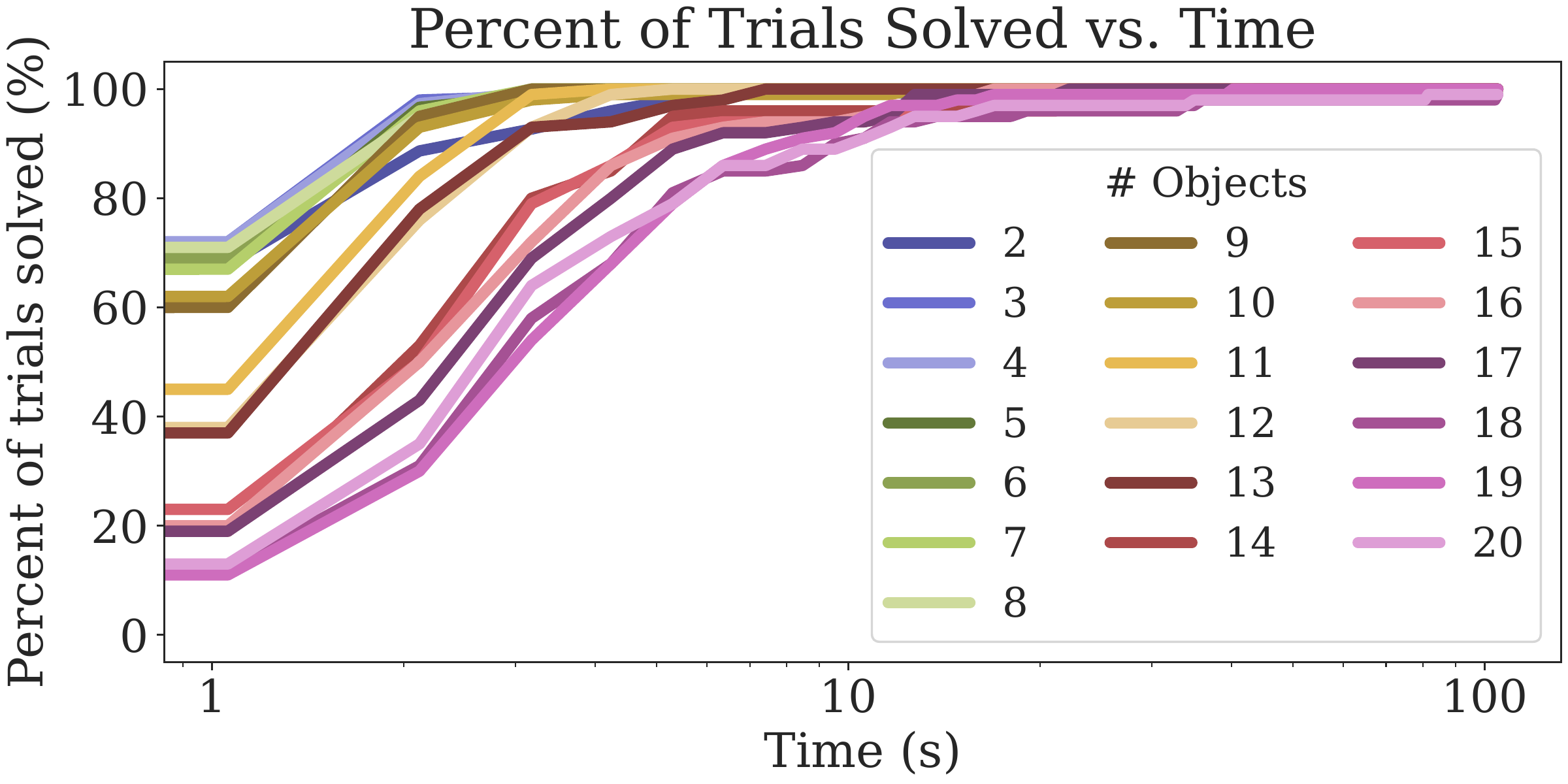}}%
	\hfill
	\subfloat[\label{fig:shelves.optimizing}]{\includegraphics[height=0.16\textwidth,width=0.33\textwidth]{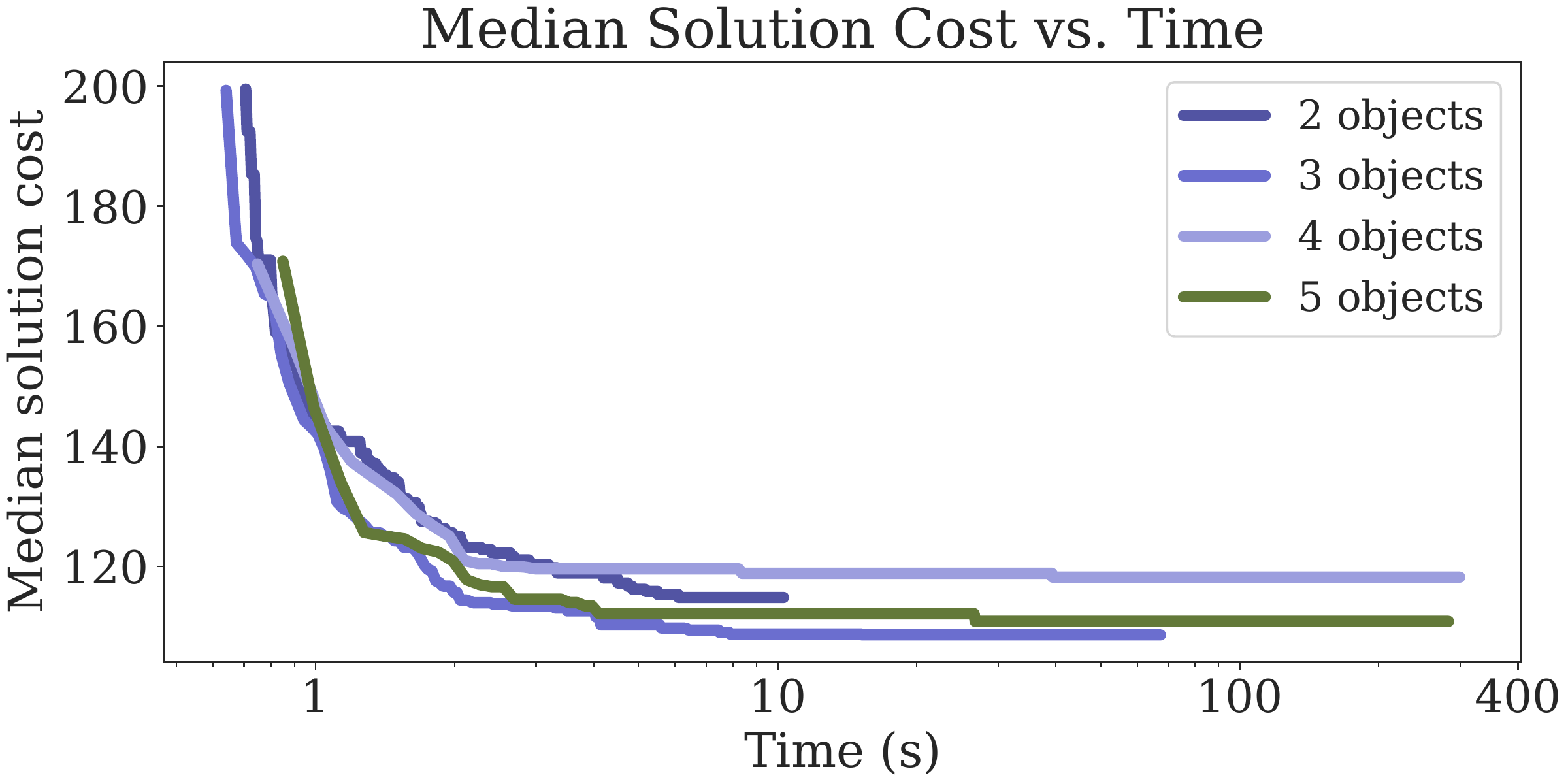}}%
	\caption{%
  Results for 100 trials of the shelf rearrangement problem.
\Cref{fig:shelves.time,fig:shelves.optimizing} show median time or cost (respectively); the error bars in~\cref{fig:shelves.time} show a 95\% confidence interval of the median.
Each trial had 300 seconds of planning time; trials which failed to find a solution within this time were counted as having infinite duration and cost.
\Cref{fig:shelves.percent} shows the cumulative percentage of problems solved as a function of time for each instance size.
}\label{fig:shelves.results}%
\end{figure*}
TMIT* is a novel approach to almost-surely asymptotically optimal TMP.\@
It extends work on constraint-based symbolic planning~\cite{dantam_incremental_constraint-based_2018}, distance-based predicate representation~\cite{thomason_unified_sampling-based_2019}, and batch-sampling-based optimal motion planning~\cite{strub_aitstar_eitstar_2021}.
TMIT* solves a relaxed symbolic planning problem with a novel SMT-based makespan-optimal symbolic planner to generate candidate sequences of actions, then attempts to find geometrically valid instantiations of these actions through asymmetric bidirectional batch-sampling-based motion planning in a hybrid multimodal state space.
It uses a differentiable distance-based representation of geometric predicates to guide parameter sampling and sample action-precondition-satisfying states through gradient-based optimization.
When candidate symbolic plans are not feasible, it generates alternatives by blocking invalid action sequence prefixes; however, it is able to continue to consider older candidate plans without backtracking by continuing the motion planning process.

Asymmetric bidirectional motion planning is well-suited to TMP because it gains information about action feasibility before paying the cost of validating a candidate plan's edges.
Incrementally improving a RGG further allows planners to reuse motion planning effort across candidate plans.
Future work may investigate using asymmetric bidirectional motion planning algorithms better suited for complex cost functions, such as Effort Informed Trees (EIT*)~\cite{strub_aitstar_eitstar_2021}.

Encoding task planning as SMT via a custom theory offers untapped potential performance improvements for TMP.\@
It provides an easy extension point for a ``theory of TMP'', incorporating geometric information such as reachability or action feasibility into the symbolic planner.
We leave exploration of this capacity for future work.

Future work could also investigate accelerating the discovery of initial solutions by explicitly biasing RGG growth toward task-relevant regions.

\printbibliography{}
\end{document}